\documentclass[oneside,a4paper,onecolumn,11pt]{article}

\usepackage{fullpage}
\usepackage[left=2cm,top=2cm,bottom=2cm,right=2cm,includehead,nomarginpar,headheight=16pt]{geometry}
\usepackage[ruled, linesnumbered, vlined, commentsnumbered]{algorithm2e}
\usepackage{graphicx} % figure related
\usepackage{amsfonts,amssymb,amsmath,amsthm,amsopn,mathtools}	% math related
\usepackage{booktabs,diagbox,colortbl,multirow,tabularx,threeparttable,hhline}
\usepackage[listings,skins,breakable]{tcolorbox}

\usepackage{fancyhdr,nopageno,lastpage} % setting header and footer

\usepackage{enumerate}
\usepackage[shortlabels]{enumitem}
\usepackage{csquotes}
\usepackage{footnote}
\usepackage{authblk}
\usepackage{hyperref}
\usepackage{prettyref}
\usepackage{setspace}
\usepackage{color}
\usepackage{xcolor}  % Required for custom colors
\usepackage{geometry}
\usepackage{academicons}
\usepackage{fontawesome5}
\usepackage{pifont,ifsym,marvosym,manfnt} % math fonts

\usepackage{tikz}
\usepackage{pgfplots}
\usetikzlibrary{positioning,shapes,shadows,arrows,calc}
\tikzstyle{component}=[rectangle, draw=black, rounded corners, fill=blue!40, drop shadow, text centered, anchor=north, text=white, minimum height=1cm]
\tikzstyle{arrow}=[->, thick]

\pgfplotsset{compat=1.12}
\usetikzlibrary{intersections}
\usetikzlibrary{pgfplots.statistics}
\usepgfplotslibrary{fillbetween}

\fancypagestyle{plain}{%
    \fancyfoot[C]{}
    \fancyfoot[R]{\faLeanpub\ \, \thepage\ / \pageref{LastPage}}
    \fancyfoot[L]{\faBraille\ \textsc{COLALab Report}}
}

\hypersetup{
    colorlinks=true,
    linkcolor=ultramarine,
    filecolor=magenta,
    urlcolor=ultramarine,
    pdftitle={Overleaf Example},
    pdfpagemode=FullScreen,
}

\definecolor{red(munsell)}{rgb}{0.95, 0.0, 0.24}
\definecolor{navyblue}{RGB}{0, 0, 128}
\definecolor{myblue}{RGB}{34,31,217}
\definecolor{mycyan}{gray}{.7}
\definecolor{Gray}{gray}{0.9}
\definecolor{usccardinal}{rgb}{0.6, 0.0, 0.0}
\definecolor{ultramarine}{RGB}{0,32,96}
\definecolor{amber}{rgb}{1.0, 0.49, 0.0}

\newtheorem{theorem}{Theorem}

\newtheorem{definition}{Definition}
\newtheorem{lemma}{Lemma}

\newtcolorbox{quotebox}{colback=gray!10,boxrule=0.4pt,colframe=black,fonttitle=\bfseries,top=1pt,bottom=1pt}

\SetCommentSty{mycommfont}

\newrefformat{fig}{Fig.~\ref{#1}}
\newrefformat{tab}{Table~\ref{#1}}
\newrefformat{sec}{Section~\ref{#1}}
\newrefformat{alg}{Algorithm~\ref{#1}}
\newrefformat{property}{Property~\ref{#1}}
\newrefformat{theorem}{Theorem~\ref{#1}}
\newrefformat{definition}{Definition~\ref{#1}}
\newrefformat{corollary}{Corollary~\ref{#1}}
\newrefformat{lemma}{Lemma~\ref{#1}}
\newrefformat{conj}{Conjecture~\ref{#1}}
\newrefformat{def}{Definition~\ref{#1}}
\newrefformat{eq}{equation~(\ref{#1})}
\newrefformat{app}{Appendix~\ref{#1}}

\usepackage{lscape}

\newenvironment{code-example}
{
\vspace{0.15cm}
\noindent\begin{minipage}{\linewidth}
\begin{center}
\arrayrulecolor{black}
\color{black}
\begin{tabular}{|p{0.95\linewidth}|}
\hline%
\rowcolor{pink!20}%
}
{
\\\hline
\end{tabular}
\end{center}
\end{minipage}
\vspace{-0.2cm}
}

\newcommand{\equalcontrib}{\textsuperscript{\dag}}
\newcommand{\corrauthor}{\textsuperscript{\ddag}}
\usepackage{subcaption}
\usepackage[numbers]{natbib}
\usepackage{algorithmic}
\begin{document}

%% title
\title{\vspace{-1ex}\LARGE\textbf{Beyond Monotonicity: Revisiting Factorization Principles in Multi-Agent Q-Learning}~\footnote{Accepted as an oral paper at AAAI 2026.}}

\author[1]{\normalsize Tianmeng Hu\thanks{These authors contributed equally.}}
\author[2]{\normalsize Yongzheng Cui\equalcontrib}
\author[2]{\normalsize Rui Tang\equalcontrib}
\author[2]{\normalsize Biao Luo\thanks{Corresponding authors.}}
\author[1]{\normalsize Ke Li\corrauthor}

\affil[1]{\normalsize Department of Computer Science, University of Exeter, U.K.}
\affil[2]{\normalsize School of Automation, Central South University, China.}

\affil[ ]{\normalsize \texttt{th743@exeter.ac.uk, \{yongzhengcui, ruitang02\}@csu.edu.cn, biao.luo@hotmail.com, k.li@exeter.ac.uk}}

\date{}
\maketitle

\vspace{-3ex}
{\normalsize\textbf{Abstract: } }Value decomposition is a central approach in multi-agent reinforcement learning (MARL), enabling centralized training with decentralized execution by factorizing the global value function into local values. To ensure individual-global-max (IGM) consistency, existing methods either enforce monotonicity constraints, which limit expressive power, or adopt softer surrogates at the cost of algorithmic complexity. In this work, we present a dynamical systems analysis of non-monotonic value decomposition, modeling learning dynamics as continuous-time gradient flow. We prove that, under approximately greedy exploration, all zero-loss equilibria violating IGM consistency are unstable saddle points, while only IGM-consistent solutions are stable attractors of the learning dynamics. Extensive experiments on both synthetic matrix games and challenging MARL benchmarks demonstrate that unconstrained, non-monotonic factorization reliably recovers IGM-optimal solutions and consistently outperforms monotonic baselines. Additionally, we investigate the influence of temporal-difference targets and exploration strategies, providing actionable insights for the design of future value-based MARL algorithms.


%!TeX root=main.tex

\section{Introduction}
\label{sec:introduction}

Cooperative learning is fundamental to enabling complex collective intelligence in multi-agent systems (MAS), with broad applicability across domains such as multi-robot coordination~\cite{long2018towards, duan2025distributional}, autonomous driving~\cite{chu2019multi, li2022metadrive, zhang2024multi, chen2024end}, and smart grid control~\cite{roesch2020smart, chung2020distributed}. Multi-agent reinforcement learning (MARL)~\cite{tan1993irl,sunehag2018vdn, lowe2017maddpg, hu2023mo} offers a unified framework for learning both cooperative and competitive behaviors in complex environments. Within this framework, the paradigm of centralized training with decentralized execution (CTDE) has emerged as a standard approach for addressing cooperative tasks~\cite{lowe2017maddpg, sunehag2018vdn}. A central component of CTDE is Value Function Factorization (VFF), which approximates the global joint action-value function $Q_{\mathrm{tot}}$ by aggregating individual agent value functions $Q_i$. This decomposition not only enables effective credit assignment but also enhances scalability in multi-agent learning~\cite{rashid2018qmix}.

The effectiveness of VFF methods hinges on satisfying the Individual-Global-Max (IGM) principle, which ensures that decentralized greedy action selections by individual agents are aligned with the globally optimal joint action. One of the earliest approaches, Value-Decomposition Networks (VDN)~\cite{sunehag2018vdn}, achieves IGM by assuming a simple additive decomposition of the global value function. However, this assumption significantly limits its representational capacity. To overcome this limitation, QMIX~\cite{rashid2018qmix} introduces a more flexible monotonicity constraint, requiring that $\frac{\partial Q_{\mathrm{tot}}}{\partial Q_i} \ge 0$, which permits complex nonlinear relationships between $Q_{\mathrm{tot}}$ and the individual $Q_i$ values. This constraint is enforced via a structured mixing network with non-negative weights. The success of QMIX marked a milestone in MARL research and established a dominant design paradigm: ensuring IGM through structured architectural constraints.

While effective in ensuring IGM, the monotonicity constraint limits the model's expressive power~\cite{son2019qtran, wang2020qplex}. To overcome this bottleneck, subsequent research has explored more expressive architectures capable of representing the full class of IGM-consistent functions. For instance, QTRAN~\cite{son2019qtran} introduces an auxiliary value function to reformulate the optimization objective, while QPLEX~\cite{wang2020qplex} proposes a sophisticated duplex dueling architecture to capture richer value structures. Although these methods expand representational power in theory, they often suffer from instability or excessive architectural complexity in practice. Meanwhile, an alternative line of work attributes the failure in non-monotonic tasks to the phenomenon of relative overgeneralization, arguing that simplistic exploration strategies such as $\epsilon$-greedy are insufficient to escape suboptimal equilibria. This has motivated the development of more elaborate coordination-aware exploration mechanisms, exemplified by methods like MAVEN~\cite{mahajan2019maven} and UneVEn~\cite{gupta2021uneven}.

However, we observe that prior work typically analyzes IGM-relevant algorithmic behavior in matrix games under a uniformly random exploration policy~\cite{wang2020qplex}. Under such settings, both monotonic and non-monotonic variants of QMIX fail to learn IGM-optimal solutions. In contrast, practical Q-learning commonly employs approximately greedy exploration strategies conditioned on the evolving value function. Motivated by this, we conducted preliminary experiments using a non-monotonic variant of QMIX combined with an $\epsilon$-greedy exploration policy on a challenging matrix game. Surprisingly, the algorithm consistently converged to IGM-consistent optimal solutions. 

This empirical finding leads us to hypothesize that the underlying learning dynamics are fundamentally altered under approximately greedy policies. Specifically, we conjecture that the feedback loop between the learned value function and the policy induces an implicit self-correcting mechanism, which naturally drives the system away from suboptimal, IGM-inconsistent solutions and toward globally optimal ones—without requiring explicit structural constraints. To validate our theoretical findings, we perform a novel dynamical‐systems analysis of non‐monotonic value‐decomposition Q‐learning. Our main contributions are:
\begin{itemize}
    \item We formulate the non‐monotonic value‐decomposition Q‐learning update as a continuous‐time gradient flow and derive its learning dynamics.
    
    \item We prove that under approximately greedy exploration policies, every zero‐loss fixed point that violates the IGM principle becomes an unstable saddle point, from which the learning trajectory naturally escapes. We further show that only IGM‐consistent solutions are stable attractors of the flow.
    
    \item Extensive experiments on matrix games and challenging multi-agent benchmarks confirm our theory. Results show that unconstrained QMIX reliably recovers IGM‐consistent optimal solutions in matrix games and outperforms monotonic baselines.
    
    \item We explore the impact of SARSA-style TD targets~\cite{rummery1994line} and intrinsic‐reward–driven exploration mechanisms on the stability and performance of non‐monotonic value decomposition.
\end{itemize}

%!TeX root=main.tex

\section{Preliminaries}
\label{sec:preliminaries}

\subsection{Dec-POMDPs}

Cooperative multi-agent tasks are commonly modeled as decentralized partially observable Markov decision processes (Dec-POMDPs)~\cite{bernstein2002complexity}. A Dec-POMDP is formally defined as a tuple
\begin{equation}
G = \langle N, S, \{\mathcal{A}_i\}, P, r, \{\Omega_i\}, O, \gamma \rangle, 
\end{equation}
where:

\begin{itemize}
    \item $N$ is the number of agents.
    \item $S$ is the global state space of the environment.
    \item $\boldsymbol{A} = \mathcal{A}_1 \times \cdots \times \mathcal{A}_n$ denotes the joint action space, where $\mathcal{A}_i$ denotes the discrete action space of agent $i$.
    \item $P(s' \mid s, \boldsymbol{a}): S \times \mathbf{\mathcal{A}} \times S \to [0, 1]$ is the state transition function.
    \item $r(s, \boldsymbol{a}): S \times \mathbf{\mathcal{A}} \to \mathbb{R}$ is the shared reward function, assigning the same team reward to all agents at each time step.
    \item $\boldsymbol{O} = O_1 \times \cdots \times O_n$ is the joint observation space. $Z: S \to \boldsymbol{O}$ is the observation function, where each agent $i$ receives a local observation $o^i_t = Z(s_t)_i$.
    \item $\gamma \in [0, 1)$ is the discount factor.
\end{itemize}

All agents act according to their individual policies $\mu = (\pi_1, \dots, \pi_N)$, with the collective objective of maximizing the expected cumulative discounted team return:
\begin{equation}
J(\pi) = \mathbb{E}_{s_0, \boldsymbol{a}_t, s_{t+1}} \left[ \sum_{t=0}^\infty \gamma^t r(s_t, \boldsymbol{a}_t) \;\middle|\; \mu \right].
\end{equation}

\subsection{Value Function Factorization}

Value function factorization (VFF) aims to learn a joint action-value function $Q_{\text{tot}}(s, \boldsymbol{a})$ for cooperative multi-agent settings. This factorization is typically implemented via a mixing function that takes as input the individual agent values $Q_i$ and, optionally, the global state $s$, and outputs the global value $Q_{\text{tot}}$. For example, VDN adopts a simple additive form:
\begin{equation}
Q_{\text{tot}}(s, \boldsymbol{a}) = \sum_{i=1}^n Q_i(\eta_i, a_i),
\end{equation}
where $\eta_i$ denotes agent $i$'s action-observation history. QMIX employs a nonlinear mixing network $g_{\text{mix}}$:
\begin{equation}
Q_{\text{tot}}(s, \boldsymbol{a}) = g_{\text{mix}}\left(\{Q_i(\eta_i, a_i)\}_{i=1}^n, s\right).
\end{equation}

These methods are trained end-to-end by minimizing the joint Bellman error. The typical loss function is defined as:
\begin{equation}
L(\theta) = \mathbb{E}_{s, \boldsymbol{a}, r, s'} \left[ \left( y^{\text{tot}} - Q_{\text{tot}}(s, \boldsymbol{a}; \theta) \right)^2 \right],
\end{equation}
where the target value is given by
\[
y^{\text{tot}} = r + \gamma \max_{\boldsymbol{a}'} Q_{\text{tot}}(s', \boldsymbol{a}'; \theta^-),
\]
and $\theta$ and $\theta^-$ denote the parameters of the current and target networks, respectively.

\subsection{IGM Consistency}

A value decomposition is said to satisfy the Individual-Global-Max (IGM) principle if, for any state $s$, the optimal joint action with respect to the global value function $Q_{\text{tot}}$ can be obtained by independently maximizing each agent's local utility function. Formally, this condition is defined as:
\begin{equation}
\begin{aligned}
&\arg\max_{\boldsymbol{a} \in \mathbf{\mathcal{A}}} Q_{\text{tot}}(s, \boldsymbol{a}) \\
&\quad= \Big(
\arg\max_{a_1 \in \mathcal{A}_1} Q_1(\eta_1, a_1),
\dots,
\arg\max_{a_n \in \mathcal{A}_n} Q_n(\eta_n, a_n)
\Big).
\end{aligned}
\end{equation}

The IGM principle is crucial because it enables decentralized execution: when the condition holds, each agent can independently select its action by greedily maximizing its own local utility $Q_i$, and the resulting joint action will be globally optimal with respect to the learned $Q_{\text{tot}}$.

\section{Analysis}
\label{sec:analysis}

We begin our analysis by simplifying the original Dec-POMDP problem into a single-state matrix game. Conceptually, a single-state game can be viewed as a localized coordination subproblem embedded within any Dec-POMDP state. Instead of using a temporal-difference target, we assume a known, fixed ground-truth payoff function $y(\boldsymbol{a})$. The formal definition is as follows:

\begin{definition}[Single-State Game]
Consider a single-state matrix game involving $N$ agents. Each agent $i \in \{1,\dots,N\}$ has a discrete action set $\mathcal{A}_i$. The joint action is denoted as $\boldsymbol{a} = (a_1, \dots, a_N)$, and the corresponding ground-truth reward is $y(\boldsymbol{a}) \in \mathbb{R}$. We assume the existence of a unique globally optimal joint action:
\begin{equation}
\boldsymbol{a}^* = \operatorname*{arg\,max}_{\boldsymbol{a}} y(\boldsymbol{a}).
\end{equation}
\end{definition}

The value function is represented by a parameterized value decomposition network. For a given parameter vector $\boldsymbol{\theta}$, the network outputs local values $Q_i(a_i; \boldsymbol{\theta})$ for each agent. A mixing network $f_{\mathrm{mix}}$ then aggregates these local values into a joint value:
\begin{equation}
Q_{\mathrm{tot}}(\boldsymbol{a}; \boldsymbol{\theta}) = f_{\mathrm{mix}}\bigl(Q_1(a_1; \boldsymbol{\theta}), \dots, Q_N(a_N; \boldsymbol{\theta}); \boldsymbol{\theta}\bigr).
\end{equation}
To facilitate theoretical analysis, we define a global state vector $\mathbf{q} \in \mathbb{R}^{\sum_i |\mathcal{A}_i|}$ that concatenates all local $Q$-values:
\begin{equation}
\begin{aligned}
\mathbf{q} := \Big( 
&Q_1(a_{1,1}), \dots, Q_1(a_{1, |\mathcal{A}_1|}), \\
&\dots, \\
&Q_N(a_{N,1}), \dots, Q_N(a_{N, |\mathcal{A}_N|}) 
\Big)^\top.
\end{aligned}
\end{equation}
From this perspective, each $Q_i(a_i)$ is a component of the vector $\mathbf{q}$, and the joint value function $Q_{\mathrm{tot}}(\boldsymbol{a}; \mathbf{q})$ is treated as a function that directly takes $\mathbf{q}$ as input, with its specific form determined by the structure of $f_{\mathrm{mix}}$. We assume that $f_{\mathrm{mix}}$ satisfies a local non-degeneracy condition: it is locally Lipschitz continuous in each coordinate direction, and its Jacobian with respect to $\mathbf{q}$ has full rank on the normal subspace at the points considered in our analysis.

The learning objective is to minimize the mean squared error loss weighted by a behavior policy $\mu(\boldsymbol{a}; \mathbf{q})$:
\begin{equation}
\mathcal{L}(\mathbf{q}) = \sum_{\boldsymbol{a}} \mu(\boldsymbol{a}; \mathbf{q}) \left( y(\boldsymbol{a}) - Q_{\mathrm{tot}}(\boldsymbol{a}; \mathbf{q}) \right)^2.
\end{equation}

\subsection{Fixed Uniform Policy}

We begin by considering a fixed, uniform behavior policy $\mu_0$ that is independent of $\mathbf{q}$, i.e., $\mu_0(\boldsymbol{a}) = 1/|\mathcal{A}|^N$. Under this setting, the loss simplifies to:
\begin{equation}
\mathcal{L}_0(\mathbf{q}) = \frac{1}{|\mathcal{A}|^N} \sum_{\boldsymbol{a}} \left( y(\boldsymbol{a}) - Q_{\mathrm{tot}}(\boldsymbol{a}; \mathbf{q}) \right)^2.
\end{equation}

This learning problem reduces to a standard supervised regression task, where the goal is to fit $Q_{\mathrm{tot}}$ to a fixed target function $y(\boldsymbol{a})$. The gradient of the loss with respect to $\mathbf{q}$ is given by:
\begin{equation}
\begin{aligned}
    &\nabla_{\mathbf{q}} \mathcal{L}_0(\mathbf{q}) \\
    &\quad = -\frac{2}{|\mathcal{A}|^N} \sum_{\boldsymbol{a}} \left( y(\boldsymbol{a}) - Q_{\mathrm{tot}}(\boldsymbol{a}; \mathbf{q}) \right) \nabla_{\mathbf{q}} Q_{\mathrm{tot}}(\boldsymbol{a}; \mathbf{q}).
\end{aligned}
\end{equation}

\begin{theorem}[Zero-Loss Points Without IGM Consistency]
Under the fixed uniform policy, there exists a set of zero-loss points $\mathcal{M}_0 = \{\mathbf{q} \mid \mathcal{L}_0(\mathbf{q}) = 0\}$ containing infinitely many elements, and this set includes points that do not satisfy IGM consistency.
\end{theorem}

\begin{proof}[Proof]
Due to the universal approximation capability of $f_{\mathrm{mix}}$, there may exist many local $Q$-value configurations $\mathbf{q}$ that violate the IGM condition but can still perfectly fit the target payoff function $y(\boldsymbol{a})$ via appropriate mixing function parameters. At such points, the loss is exactly zero, and the gradient vanishes accordingly. The inherent ambiguity in credit assignment leads to a multiplicity of solutions. Consequently, the learning process may converge to any of these zero-loss solutions without any inherent bias toward IGM-consistent ones.
\end{proof}

\subsection{Approximately Greedy Policy}

When using an approximately greedy policy, the behavior policy depends on the current value estimates, resulting in a value-coupled dynamical system. Directly analyzing $\epsilon$-greedy policies is challenging due to the non-differentiability of the $\arg\max$ operator. To circumvent this issue, we introduce a smooth and differentiable surrogate based on the softmax policy with temperature parameter $\tau > 0$:
\begin{align}
\pi_i^{\tau}(a_i \mid \mathbf{q}) &:= \frac{e^{Q_i(a_i)/\tau}}{\sum_{b \in \mathcal{A}_i} e^{Q_i(b)/\tau}}, \\
\mu_{\tau}(\boldsymbol{a} \mid \mathbf{q}) &= \prod_{i=1}^N \pi_i^{\tau}(a_i \mid \mathbf{q}).
\end{align}

The corresponding smooth loss function is defined as:
\begin{equation}
\mathcal{L}_{\tau}(\mathbf{q}) = \sum_{\boldsymbol{a}} \mu_{\tau}(\boldsymbol{a} \mid \mathbf{q}) \left( y(\boldsymbol{a}) - Q_{\mathrm{tot}}(\boldsymbol{a}; \mathbf{q}) \right)^2.
\end{equation}

\begin{lemma}\label{lem:smooth_to_clarke}
Let $\{\tau_k\}_{k=1}^{\infty}$ be a sequence of temperatures such that $\tau_k \downarrow 0$. Suppose that at parameter point $\mathbf{q}$, each agent has a unique greedy action and the mixing function $f_{\mathrm{mix}}$ satisfies the local non-degeneracy condition. Then,
\begin{equation}
\lim_{k \to \infty} \nabla_{\mathbf{q}} \mathcal{L}_{\tau_k}(\mathbf{q}) \;\in\; \partial^{\mathrm{C}}_{\mathbf{q}} \mathcal{L}_0(\mathbf{q}),
\end{equation}
where $\partial^{\mathrm{C}}$ denotes the Clarke generalized gradient, and $\mathcal{L}_0$ corresponds to the limiting loss as $\tau \to 0$ in the softmax policy $\mu_\tau$.
\end{lemma}
Lemma~\ref{lem:smooth_to_clarke} provides a theoretical justification for using the smooth surrogate system to analyze the stability properties of the original non-smooth objective. The complete proof is deferred to Appendix A.

In contrast to the fixed uniform policy case, the gradient of $\mathcal{L}_\tau(\mathbf{q})$ now consists of two distinct components:
\begin{equation}
\begin{aligned}
\nabla_{\mathbf{q}} \mathcal{L}_\tau(\mathbf{q}) 
= &\underbrace{
\sum_{\boldsymbol{a}} 
\left( \nabla_{\mathbf{q}} \mu_\tau(\boldsymbol{a} \mid \mathbf{q}) \right)
\left( y(\boldsymbol{a}) - Q_{\mathrm{tot}} \right)^2
}_{\text{policy gradient term}} \\
&\quad - \underbrace{
2 \sum_{\boldsymbol{a}} 
\mu_\tau(\boldsymbol{a} \mid \mathbf{q}) 
\left( y(\boldsymbol{a}) - Q_{\mathrm{tot}} \right) 
\nabla_{\mathbf{q}} Q_{\mathrm{tot}}
}_{\text{value gradient term}}.
\end{aligned}
\end{equation}
The first term (the policy gradient term) alters the learning dynamics compared to the uniform case. We model the learning process as a continuous-time gradient flow, represented by the following ordinary differential equation (ODE):
\begin{equation}
\dot{\mathbf{q}} = -\nabla_{\mathbf{q}} \mathcal{L}_{\tau}(\mathbf{q}).
\end{equation}

The local stability of a fixed point $\mathbf{q}^*$ is determined by the spectrum of its Jacobian matrix, given by $J_{\tau}(\mathbf{q}^*) = -H_{\tau}(\mathbf{q}^*)$, where $H_{\tau} = \nabla_{\mathbf{q}}^2 \mathcal{L}_{\tau}$ is the Hessian matrix of the loss. If $H_{\tau}(\mathbf{q}^*)$ is positive definite, then $\mathbf{q}^*$ is asymptotically stable. If it has at least one negative eigenvalue, then $\mathbf{q}^*$ is unstable.

We now focus on the set of all zero-loss points and analyze their stability. These points form a \emph{zero-loss manifold} $\mathcal{M}$, defined as the set of parameter configurations that perfectly fit the ground-truth payoff:
\begin{equation}
\mathcal{M} := \left\{ \mathbf{q} \mid Q_{\mathrm{tot}}(\boldsymbol{a}; \mathbf{q}) = y(\boldsymbol{a}), \;\forall\, \boldsymbol{a} \right\}.
\end{equation}
Every point within $\mathcal{M}$ is a fixed point of the gradient flow. The local stability of such points is governed by the curvature of the loss function, as characterized by the Hessian matrix. Theorem~\ref{thm:igm_stable} establishes that for any perturbation direction that alters a local greedy action to a suboptimal one, the corresponding quadratic form is strictly positive. As a result, the fixed point $\mathbf{q}^*$ is locally asymptotically stable along the center-stable manifold.

\begin{theorem}[Stability of IGM-Consistent Fixed Points]\label{thm:igm_stable}
Suppose the following conditions hold:

(H1) For every $\mathbf{q} \in \mathcal{M}$, each agent's greedy action $\arg\max_{a_i} Q_i(a_i)$ is unique.

(H2) The Jacobian of the mixing function $f_{\mathrm{mix}}$ has full rank on the normal subspace of $\mathcal{M}$ (local non-degeneracy).

(H3) The softmax temperature $\tau > 0$ is sufficiently small such that the induced policy is dominated by greedy actions.

Then for any zero-loss fixed point $\mathbf{q}^* \in \mathcal{M}$ that satisfies IGM consistency (i.e., $\mathbf{g}(\mathbf{q}^*) = \mathbf{u}^*$), the Hessian $H_\tau(\mathbf{q}^*)$ is positive definite on the normal subspace $(T_{\mathbf{q}^*} \mathcal{M})^\perp$.
\end{theorem}

\begin{proof}[Proof]
We determine the spectral properties of the Hessian $H_\tau(\mathbf{q}^*)$ by analyzing the sign of its quadratic form $\mathbf{v}^\top H_\tau(\mathbf{q}^*) \mathbf{v}$ along critical directions. The full proof is provided in Appendix A.

We aim to show that $H_\tau(\mathbf{q}^*)$ is positive definite on the normal subspace $(T_{\mathbf{q}^*} \mathcal{M})^\perp$. Consider an arbitrary perturbation direction $\mathbf{v}$ in this subspace, such as $\mathbf{v} = \mathbf{e}_{i u_i'} - \mathbf{e}_{i u_i^*}$, where $u_i' \neq u_i^*$. This perturbation attempts to flip agent $i$’s greedy action from the optimal $u_i^*$ to a suboptimal alternative $u_i'$.

On the zero-loss manifold $\mathcal{M}$, the quadratic form of the Hessian in this direction can be computed explicitly as:
\begin{equation}\label{eq:hessian_form_stable}
\begin{aligned}
&\mathbf{v}^\top H_\tau(\mathbf{q}^*) \mathbf{v} \\
&= \frac{1}{\tau} \sum_{\boldsymbol{a}} \mu_\tau(\boldsymbol{a} \mid \mathbf{q}^*) 
\Big[\, \big( \mathbf{1}[a_i = u_i'] - \pi_i^\tau(u_i') \big) \\
&\quad - \big( \mathbf{1}[a_i = u_i^*] - \pi_i^\tau(u_i^*) \big) \,\Big]^2.
\end{aligned}
\end{equation}
In the low-temperature limit $\tau \to 0$, the policy $\mu_\tau(\boldsymbol{a} \mid \mathbf{q}^*)$ concentrates its mass on the greedy joint action $\mathbf{u}^*$. For this dominant term $\boldsymbol{a} = \mathbf{u}^*$, we have $\pi_i^\tau(u_i') \to 0$, $\pi_i^\tau(u_i^*) \to 1$, and the indicator functions yield $\mathbf{1}[a_i = u_i'] = 0$, $\mathbf{1}[a_i = u_i^*] = 1$. The expression inside the brackets thus converges to a nonzero constant.
Since the quadratic form is a weighted sum of squared terms with the dominant contribution strictly positive, we conclude that $\mathbf{v}^\top H_\tau(\mathbf{q}^*) \mathbf{v} > 0$. Therefore, $H_\tau(\mathbf{q}^*)$ is positive definite in all directions that perturb the optimal greedy action, ensuring that $\mathbf{q}^*$ is asymptotically stable.
\end{proof}

\begin{theorem}[Instability of IGM-Inconsistent Fixed Points]\label{thm:igm_saddle}
Under the same assumptions (H1) -- (H3) as in Theorem~\ref{thm:igm_stable}, consider a zero-loss fixed point $\mathbf{q}^* \in \mathcal{M}$ that violates IGM consistency (i.e., $\mathbf{g}(\mathbf{q}^*) \neq \mathbf{u}^*$).
Then there exists an agent $i$ and a perturbation direction
\begin{equation}
\mathbf{v} = \mathbf{e}_{i, u_i^*} - \mathbf{e}_{i, g_i(\mathbf{q}^*)}
\end{equation}
such that
\begin{equation}
\mathbf{v}^\top H_\tau(\mathbf{q}^*) \mathbf{v}
\approx -\frac{2}{\tau} \left[ y(\mathbf{u}^*) - y(\mathbf{g}(\mathbf{q}^*)) \right] < 0,
\end{equation}
indicating that $H_\tau(\mathbf{q}^*)$ has a negative eigenvalue, and $\mathbf{q}^*$ is a structurally unstable saddle point.
\end{theorem}

\begin{proof}[Proof]
If $\mathbf{g}(\mathbf{q}^*) \neq \mathbf{u}^*$, we aim to show that $H_\tau(\mathbf{q}^*)$ admits at least one negative eigenvalue. Let $g_i = g_i(\mathbf{q}^*)$ denote the suboptimal greedy action selected by agent $i$. We construct a perturbation direction intended to "correct" this suboptimal choice:
\begin{equation}
\mathbf{v} = \mathbf{e}_{i u_i^*} - \mathbf{e}_{i g_i}.
\end{equation}
This perturbation attempts to redirect the greedy policy toward an action with potentially higher reward.

In the low-temperature regime, the softmax policy induces a gradient amplification effect that makes the Hessian’s curvature primarily governed by reward differences. It can be shown that the quadratic form in this direction approximately satisfies:
\begin{equation}\label{eq:hessian_form_unstable}
\mathbf{v}^\top H_\tau(\mathbf{q}^*) \mathbf{v} \approx -\frac{C}{\tau} \left[ y(\mathbf{g}') - y(\mathbf{g}(\mathbf{q}^*)) \right],
\end{equation}
where $C > 0$ is a constant, and $\mathbf{g}'$ is the new greedy joint action induced by the perturbation.
Since $\mathbf{g}(\mathbf{q}^*)$ is not globally optimal, there always exists a local improving direction such that $y(\mathbf{g}') > y(\mathbf{g}(\mathbf{q}^*))$. This implies that the quadratic form is strictly negative in the direction $\mathbf{v}$, i.e., $\mathbf{v}^\top H_\tau(\mathbf{q}^*) \mathbf{v} < 0$, thereby confirming the existence of a negative eigenvalue in the Hessian. Consequently, the fixed point $\mathbf{q}^*$ is a saddle point.
\end{proof}

Combining Theorems~\ref{thm:igm_stable} and~\ref{thm:igm_saddle}, we conclude that the set of IGM-consistent zero-loss solutions forms the unique center-stable manifold of the learning dynamics, while any IGM-inconsistent zero-loss point is a saddle point and hence unstable. Under continued exploration, the system is guaranteed to escape from these unstable saddle points and ultimately converge to the stable submanifold: 
\begin{equation}
\mathcal{M}_{\mathrm{IGM}} = \left\{ \mathbf{q} \in \mathcal{M} \mid \mathbf{g}(\mathbf{q}) = \mathbf{u}^* \right\}.
\end{equation}

\section{Methodology}

In this section, we investigate a non-monotonic variant of the QMIX algorithm and its two extensions designed to address the challenges introduced by removing the monotonicity constraint. Additional implementation details, including hyperparameters and pseudocode, are provided in Appendix~B.

\subsection{Non-Monotonic Mixing Function}
We investigate a non-monotonic variant of the QMIX algorithm. Standard QMIX imposes a monotonicity constraint on the mixing network, requiring that the joint action-value function $Q_{\mathrm{tot}}$ be non-decreasing with respect to each individual agent's utility $Q_i$. Formally, $\frac{\partial Q_{\mathrm{tot}}}{\partial Q_i} \geq 0$, which is enforced by constraining the mixing network weights to be non-negative. While this constraint guarantees IGM consistency by design, it restricts the expressive capacity of the model.

In contrast, our approach removes this structural constraint, allowing the mixing network to learn an arbitrary aggregation function of the individual values:
\begin{equation}
    Q_{\mathrm{tot}}(s, \mathbf{a}) = g_{\mathrm{mix}}\big(\{Q_i(\eta_i, a_i)\}_{i=1}^n, s\big),
\end{equation}
where $s$ denotes the global state, $\eta_i$ represents the local observation history of agent $i$, and $g_{\mathrm{mix}}$ is implemented by a feed-forward neural network. Apart from the removal of the monotonicity constraint, the architecture remains identical to QMIX.

Our central hypothesis is that under approximately greedy exploration policies, the learning dynamics are sufficient to drive the system towards IGM-consistent solutions, even in the absence of explicit architectural constraints. The network is trained end-to-end by minimizing the joint Bellman error:
\begin{equation}
    L(\theta) = \mathbb{E}_{s, \mathbf{a}, r, s'} \big[ \big( y^{\mathrm{tot}} - Q_{\mathrm{tot}}(s, \mathbf{a}; \theta) \big)^2 \big],
\end{equation}
where the target is computed using a target network,
\begin{equation}
    y^{\mathrm{tot}} = r + \gamma \max_{\mathbf{a}'} Q_{\mathrm{tot}}(s', \mathbf{a}'; \theta^-).
\end{equation}

\subsection{SARSA-Style Updating with TD($\lambda$)}
Removing the monotonicity constraint introduces a potential issue: during the learning process, IGM consistency can no longer be guaranteed prior to convergence. In standard Q-learning, this may lead to undesirable effects because the target
\begin{equation}
    y^{\mathrm{tot}}_{\mathrm{Q}} 
    = r + \gamma \max_{\mathbf{a}'} Q_{\mathrm{tot}}(s', \mathbf{a}'; \theta^-),
\end{equation}
which relies on the $\max_{\mathbf{a}'}$ operator, is no longer a reliable training signal in the absence of IGM.

To address this issue, we adopt a SARSA-style update rule that removes the problematic $\max$ operator. Instead of computing the target based on the maximum over all possible next actions, we use the joint action $\mathbf{a}'$ that was actually sampled from the replay buffer:
\begin{equation}
    y^{\mathrm{tot}}_{\mathrm{SARSA}}
    = r + \gamma Q_{\mathrm{tot}}(s', \mathbf{a}'; \theta^-),
\end{equation}
where $\mathbf{a}'$ is the next joint action executed by the behavior policy. This formulation aligns the training signal with the policy's actual behavior and mitigates the bias caused by the $\max$ operator. We then incorporate the TD($\lambda$) algorithm:
\begin{equation}
    y^{\mathrm{tot}}_{\lambda}
    = (1-\lambda) \sum_{n=1}^\infty \lambda^{n-1} y^{\mathrm{tot}}_{(n)},
\end{equation}
where $y^{\mathrm{tot}}_{(n)}$ denotes the $n$-step return. By averaging over multiple step returns, TD($\lambda$) smooths the learning signal, mitigates variance, and improves credit assignment over longer horizons, which is particularly beneficial for non-monotonic value decomposition.

It is worth noting that SARSA-type updates are theoretically on-policy and would typically require off-policy corrections, such as importance sampling. However, prior work by \citet{hernandez2019understanding} has shown that, in the context of deep RL, applying such corrections to SARSA often degrades performance. Therefore, we omit off-policy corrections in our implementation.

\subsection{Intrinsic Reward Driven Exploration-Exploitation}
In our preliminary experiments, we observed that simply increasing the exploration rate in the original QMIX algorithm provided no benefit. In contrast, for our non-monotonic variant of QMIX, a higher degree of exploration consistently improved performance. Motivated by this observation, we integrate a curiosity-driven exploration mechanism based on Random Network Distillation (RND)\cite{rnd}. RND encourages exploration by assigning intrinsic rewards for visiting novel states. It employs two neural networks: (i) a fixed, randomly initialized target network that maps states to feature embeddings, and (ii) a predictor network trained to approximate the target network’s output for visited states. The prediction error of the predictor network then serves as the intrinsic reward signal.  The total reward used to train the value function is given by:
\begin{equation}
    r = r_{\mathrm{ext}} + \beta \cdot r_{\mathrm{int}},
\end{equation}
where $r_{\mathrm{ext}}$ is the extrinsic team reward from the environment, $r_{\mathrm{int}}$ is the intrinsic curiosity reward generated by RND. This curiosity-driven mechanism serves as a more sophisticated alternative to $\epsilon$–greedy exploration and allows us to investigate the role of exploration in non-monotonic value factorization.

\section{Experiments}

We evaluate our approach on three representative benchmarks: one-step matrix games, the StarCraft Multi-Agent Challenge (SMAC) \cite{samvelyan2019smacv1}, and Google Research Football (GRF) \cite{grf}. 
We compare our method against three strong baselines: QMIX, QPLEX, and QTRAN. 
These baselines were selected to cover the major paradigms in value factorization: a strictly monotonic approach (QMIX), a refined architecture that preserves a relaxed form of monotonicity (QPLEX), and a non-monotonic method based on an alternative factorization principle (QTRAN). 
Complete experimental results, along with additional ablation studies, are provided in Appendix~C.

\begin{table*}[htbp]
\centering
\begin{tabular}{@{}c@{\hspace{3em}}c@{\hspace{3em}}c@{}}
  \begin{subtable}[t]{0.19\textwidth}
    \centering
    \resizebox{\linewidth}{!}{
      \begin{tabular}{|c|c|c|c|}
        \hline
        \diagbox{$u_1$}{$u_2$} & A & B & C \\ \hline
        A & 12 & -12 & -12 \\ \hline
        B & -12 & 0 & 0 \\ \hline
        C & -12 & 0 & 0 \\ \hline
      \end{tabular}
    }
    \caption{Payoff of Game A}
  \end{subtable}
  &
  \begin{subtable}[t]{0.26\textwidth}
    \centering
    \resizebox{\linewidth}{!}{
      \begin{tabular}{|c|c|c|c|}
        \hline
        \diagbox{$Q_1$}{$Q_2$} & A(7.4) & B(-4.0) & C(-4.1) \\ \hline
        A(8.1) & 12.0 & -12.0 & -12.0 \\ \hline
        B(-4.3) & -12.0 & 0.0 & 0.0 \\ \hline
        C(-4.3) & -12.0 & 0.0 & 0.0 \\ \hline
      \end{tabular}
    }
    \caption{Our method on Game A}
  \end{subtable}
  &
  \begin{subtable}[t]{0.26\textwidth}
    \centering
    \resizebox{\linewidth}{!}{
      \begin{tabular}{|c|c|c|c|}
        \hline
        \diagbox{$Q_1$}{$Q_2$} & A(-5.2) & B(0.1) & C(0.1) \\ \hline
        A(-5.3) & -12.0 & -12.0 & -12.0 \\ \hline
        B(0.1) & -12.0 & 0.0 & -0.1 \\ \hline
        C(0.1) & -12.0 & -0.0 & -0.0 \\ \hline
      \end{tabular}
    }
    \caption{QMIX on Game A}
  \end{subtable}
  \\[2em]
  \begin{subtable}[t]{0.19\textwidth}
    \centering
    \resizebox{\linewidth}{!}{
      \begin{tabular}{|c|c|c|c|}
        \hline
        \diagbox{$u_1$}{$u_2$} & A & B & C \\ \hline
        A & 12 & 0 & 10 \\ \hline
        B & 0 & 0 & 10 \\ \hline
        C & 10 & 10 & 10 \\ \hline
      \end{tabular}
    }
    \caption{Payoff of Game B}
  \end{subtable}
  &
  \begin{subtable}[t]{0.26\textwidth}
    \centering
    \resizebox{\linewidth}{!}{
      \begin{tabular}{|c|c|c|c|}
        \hline
        \diagbox{$Q_1$}{$Q_2$} & A(6.0) & B(-3.7) & C(2.5) \\ \hline
        A(6.2) & 12.0 & 0.0 & 10.0 \\ \hline
        B(-4.4) & 0.0 & 0.0 & 10.0 \\ \hline
        C(1.7) & 10.0 & 10.0 & 10.0 \\ \hline
      \end{tabular}
    }
    \caption{Our method on Game B}
  \end{subtable}
  &
  \begin{subtable}[t]{0.26\textwidth}
    \centering
    \resizebox{\linewidth}{!}{
      \begin{tabular}{|c|c|c|c|}
        \hline
        \diagbox{$Q_1$}{$Q_2$} & A(-0.6) & B(-0.6) & C(0.5) \\ \hline
        A(-0.6) & 0.1 & 0.1 & 4.1 \\ \hline
        B(-0.6) & 0.1 & 0.1 & 4.1 \\ \hline
        C(0.5) & 5.6 & 5.6 & 10.0 \\ \hline
      \end{tabular}
    }
    \caption{QMIX on Game B}
  \end{subtable}
\end{tabular}

\caption{True payoff matrices and estimated value functions for two matrix games. Each row corresponds to one game: the left column shows the ground-truth payoff matrix, the middle column shows value estimates from our non-monotonic QMIX variant, and the right column shows estimates from standard QMIX. Top row: Game A. Bottom row: Game B.}
\label{tab:matrix_game}
\end{table*}

\subsection{Matrix Game}
We begin by evaluating non-monotonic QMIX with SARSA-style updates on two one-step matrix games, which are adapted from \cite{son2019qtran}. Among them, Game~B is more challenging than Game~A, as it contains stronger local optima that are more difficult to escape. In addition, unlike previous studies on matrix games that employ uniformly random policies, we adopt an $\epsilon$-greedy strategy, which is consistent with our theoretical analysis. 

The results, summarized in Table~\ref{tab:matrix_game}, show that our method successfully learns the correct global payoff matrix. In contrast, QMIX fails to recover the true payoffs. These findings provide initial empirical evidence that, under appropriate exploration–exploitation mechanisms, removing the monotonicity constraint enables the algorithm to accurately recover IGM-optimal solutions in non-monotonic settings.

\subsection{SMAC and GRF}
In order to validate our propositions, we implemented RND to generate intrinsic exploration rewards for non-monotonic QMIX with SARSA-style updates, and conducted experiments in complex, high-dimensional SMAC and GRF environments. As shown in Figures \ref{fig:smaccompare} and \ref{fig:grfcompare}, our method demonstrates clear advantages over the baselines across a variety of maps. QTRAN, despite its success in the matrix game, performs poorly in these complex environments. Similarly, QMIX and QPLEX struggle to find optimal solutions in these non-monotonic sequential decision-making problems due to their inherent structural constraints. Further ablation studies on additional variants of our method are provided in Appendix~C.

\begin{figure*}[tbp]
\centering
\includegraphics[width=1.0\textwidth]{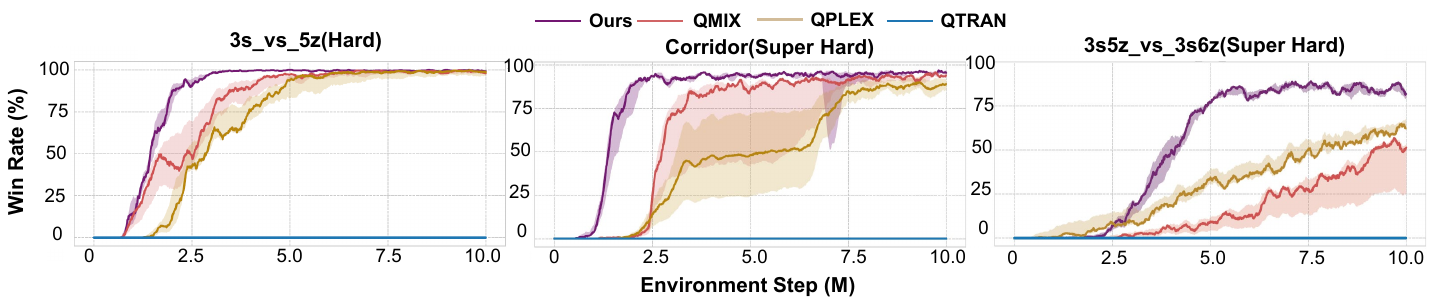}
\caption{Comparisons of test win rate on SMAC maps: \texttt{3s\_vs\_5z}, \texttt{corridor}, and \texttt{3s5z\_vs\_3s6z}. The results are averaged over five independent runs, with the 25\%–75\% interquartile range shown as a shaded region.}
\label{fig:smaccompare}
\end{figure*}

\begin{figure*}[tbp]
\centering
\includegraphics[width=1.0\textwidth]{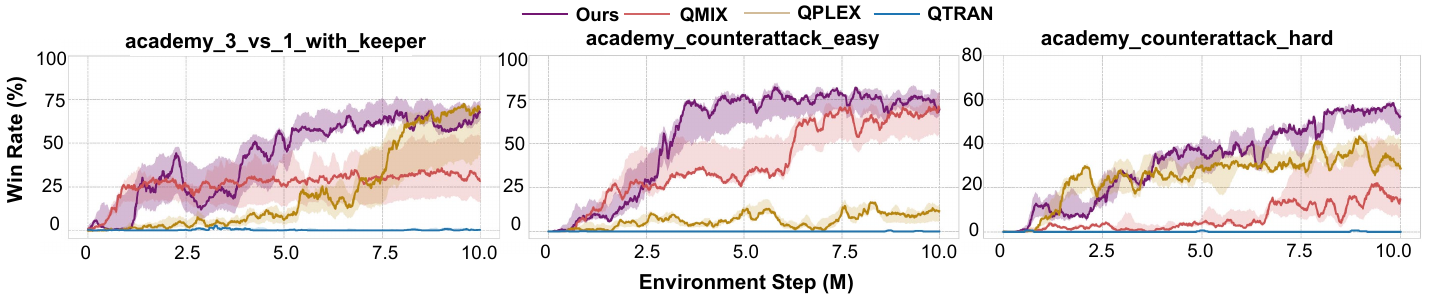}
\caption{Comparisons of test win rate on GRF tasks: \texttt{academy\_3\_vs\_1\_keeper}, \texttt{academy\_counterattack\_easy}, and \texttt{academy\_counterattack\_hard}. The results are averaged over five independent runs.}
\label{fig:grfcompare}
\end{figure*}

\subsection{Discussion}
The empirical results provide strong evidence supporting our theoretical analysis. The observed learning dynamics closely align with our model of stable and unstable equilibria. Unlike QTRAN, whose projection mechanism can restrict the use of non-optimal global actions and thus impede learning, our unconstrained approach leverages broad exploration. This enables the value function to be estimated across a wider range of joint actions, mitigating the risk of premature convergence to suboptimal solutions caused by relative overgeneralization.

We initially expected non-monotonic QMIX to perform comparably to the original QMIX. Surprisingly, however, on several challenging SMAC tasks, our method not only consistently outperforms the original QMIX but also converges significantly faster. We hypothesize that this advantage arises from the removal of the monotonicity constraint, which substantially enhances the expressiveness of the mixing network. Under an appropriate exploration mechanism, this increased flexibility allows non-monotonic QMIX to more effectively discover optimal policies. In addition, both QPLEX and QTRAN enhance the representational capacity of the value function through specialized network architectures and additional soft constraints. However, on the challenging tasks we evaluated, neither method demonstrated superior performance. We hypothesize that the increased algorithmic complexity of these approaches may reduce their robustness and heighten sensitivity to hyperparameter choices, thereby limiting their practical effectiveness compared to our simpler unconstrained formulation.

It is worth noting that in the GRF tasks, our method exhibits a slower reward increase during the early stage of training, followed by a sharp improvement in the later stage. We attribute this behavior to the existence of multiple zero-loss solutions in the initial learning dynamics, including suboptimal ones. During this exploratory phase, the algorithm navigates the policy space and transiently approaches unstable saddle points. The system eventually escapes these unstable regions and converges to the stable, IGM-consistent submanifold, resulting in a rapid performance surge and superior final results.

\section{Related Work}

Value Function Factorization (VFF) is a central branch of the CTDE paradigm in multi-agent reinforcement learning. VFF methods seek to learn a global joint action-value function $Q_{\mathrm{tot}}$ and decompose it into per-agent value functions $Q_i$. This factorized structure is designed to ensure that optimizing each local utility $Q_i$ leads to the optimization of the global objective $Q_{\mathrm{tot}}$, while enabling decentralized execution based solely on local observations.

One of the earliest VFF methods, Value-Decomposition Networks (VDN)\cite{sunehag2018vdn}, assumes that the global joint action-value function is the sum of individual agent values.  This additive assumption simplifies credit assignment but severely limits representational capacity, failing to capture nonlinear agent interactions. QMIX~\cite{rashid2018qmix} extends VDN by introducing a nonlinear mixing network, thereby enhancing expressiveness. QMIX imposes a constraint that requires $Q_{\mathrm{tot}}$ to be a monotonically non-decreasing function with respect to each individual $Q_i$. Although more expressive than VDN, this constraint still restricts the representation of tasks with inherently non-monotonic value functions.

To overcome the expressiveness limitations imposed by the monotonicity constraint, subsequent research has explored two main directions.
The first focuses on enhancing the mixing network. WQMIX~\cite{rashid2020wqmix} introduces a weighting mechanism in the projection step of monotonic value decomposition, emphasizing more promising joint actions and biasing learning toward better solutions. Qatten~\cite{yang2020qatten} applies multi-head attention to adaptively learn interaction weights among agents, enabling a more flexible nonlinear combination.

Another line aims to develop more general frameworks capable of representing the full class of IGM functions. QTRAN~\cite{son2019qtran} introduces a transformed surrogate value function, reformulating the IGM condition as a set of linear constraints and incorporating them via soft regularization. While QTRAN is theoretically more expressive, it often suffers from instability during training and performs suboptimally in practice. QPLEX~\cite{wang2020qplex} introduces a duplex dueling architecture, which factorizes $Q_{\mathrm{tot}}$ into a state-value term $V(s)$ and a joint advantage function $A(s, \mathbf{u})$. The IGM constraint is enforced by imposing monotonicity only on the advantage component. QPLEX theoretically covers the complete IGM function class and achieves strong empirical performance, though at the cost of increased architectural complexity.

Current research in VFF largely focuses on designing increasingly complex network architectures and decomposition schemes to enhance representational capacity, aiming to capture a broader range of tasks while preserving IGM guarantees. In contrast, our work offers a different perspective: we argue that under standard approximately greedy exploration strategies, the learning dynamics may exhibit an implicit stabilizing mechanism. This mechanism can naturally guide unconstrained learning process toward IGM‐consistent optimal solutions.

\section{Conclusion}
This paper challenges the prevailing assumption that structural monotonicity is necessary for multi-agent value decomposition to ensure IGM consistency. We model the non-monotonic learning process as a continuous-time gradient flow and theoretically demonstrate that, under approximately greedy exploration, the dynamics themselves provide an implicit self-correction mechanism. This mechanism drives IGM-inconsistent solutions to unstable saddle points, while establishing IGM-consistent solutions as stable attractors. Experiments on synthetic matrix games, SMAC, and GRF empirically validate our theory, showing that removing monotonicity constraints not only reliably recovers optimal solutions but also consistently outperforms monotonic baselines. Although our formal analysis is developed in the single-state setting, these findings suggest that leveraging the natural dynamics of learning rather than imposing rigid architectural constraints offers a promising direction for designing more expressive, and more effective value-based MARL algorithms.

\section*{Acknowledgment}
This work was supported in part by the UKRI Future Leaders Fellowship under Grant MR/S017062/1 and MR/X011135/1; in part by NSFC under Grant 62376056 and 62076056; in part by the Royal Society under Grant IES/R2/212077; in part by the EPSRC under Grant 2404317; in part by the Kan Tong Po Fellowship (KTP\textbackslash R1\textbackslash 231017); and in part by the Amazon Research Award and Alan Turing Fellowship.

\bibliographystyle{IEEEtran}
\bibliography{IEEEabrv,your_bib}

\newpage

\appendix
%!TeX root=main.tex

% Appendix A
\section{Proofs}
\label{app: proofs}

\subsection{Proof of Lemma~\ref{lem:smooth_to_clarke}}
\label{app:lemma_proof}

\begin{lemma}\label{lem:smooth_to_clarke}
Let $\{\tau_k\}_{k=1}^{\infty}$ be a sequence of temperatures such that $\tau_k \downarrow 0$. Suppose that at parameter point $\mathbf{q}$, each agent has a unique greedy action and the mixing function $f_{\mathrm{mix}}$ satisfies the local non-degeneracy condition. Then,
\begin{equation}
\lim_{k \to \infty} \nabla_{\mathbf{q}} \mathcal{L}_{\tau_k}(\mathbf{q}) \;\in\; \partial^{\mathrm{C}}_{\mathbf{q}} \mathcal{L}_0(\mathbf{q}),
\end{equation}
where $\partial^{\mathrm{C}}$ denotes the Clarke generalized gradient, and $\mathcal{L}_0$ corresponds to the limiting loss as $\tau \to 0$ in the softmax policy $\mu_\tau$.
\end{lemma}

\begin{proof}

For any fixed $\tau>0$, define the policy of agent $i$ as
\begin{equation}
\pi_i^\tau(a_i;\mathbf q)
=\frac{\exp(Q_i(a_i)/\tau)}{\sum_b \exp(Q_i(b)/\tau)}.
\end{equation}
$\pi_i^\tau$ is infinitely differentiable with respect to $\mathbf q$. Consequently, the joint policy 
\begin{equation}
\mu_\tau(\boldsymbol{a};\mathbf q) \;=\; \prod_i \pi_i^\tau(a_i;\mathbf q),
\end{equation}
and the corresponding loss function 
\begin{equation}
\mathcal L_{\tau}(\mathbf q) \;=\; \sum_{\boldsymbol{a}} \mu_\tau(\boldsymbol{a};\mathbf q) \big(y - Q_{\mathrm{tot}}\big)^2
\end{equation}
are both smooth and differentiable.  

In the limit $\tau \to 0$, since each agent's greedy action $g_i(\mathbf q)$ is unique, we have
\begin{equation}
\pi_i^\tau(g_i(\mathbf q); \mathbf q) \to 1, \quad 
\pi_i^\tau(a_i; \mathbf q) \to 0\quad (a_i \neq g_i(\mathbf q)).
\end{equation}
Hence, $\mu_\tau(\boldsymbol{a};\mathbf q)$ converges pointwise to the deterministic greedy policy.

The Clarke generalized gradient $\partial^{\mathrm{C}} f(\mathbf q)$ of a non-smooth function $f$ is defined as the convex hull of all limit points of gradients of smooth approximations converging to $f$:
\begin{equation}
\partial^C f(\mathbf q) \;=\; \mathrm{conv}\left\{ \lim_{k \to \infty} \nabla f_k(\mathbf q) \;\middle|\; f_k \to f,~ f_k \text{ smooth} \right\}.
\end{equation}
If a sequence of smooth functions $\{f_k\}$ converges pointwise to $f$ and has uniformly bounded Lipschitz constants, then
\begin{equation}
\lim_{k \to \infty} \nabla f_k(\mathbf q) \;\in\; \partial^C f(\mathbf q).
\end{equation}
This property directly connects the gradients of smooth approximations to the Clarke generalized gradient of the limiting non-smooth function.

We now apply this result to our problem.
Define $f_k(\mathbf q) = \mathcal L_{\tau_k}(\mathbf q)$, where $\tau_k \downarrow 0$ and the Lipschitz constants of $\{\mathcal L_{\tau_k}\}$ remain uniformly bounded. Clearly, $f_k$ converges pointwise to
\begin{equation}
f_0(\mathbf q)
= \sum_{\boldsymbol{a}} \mu_{\mathrm{greedy}}(\boldsymbol{a}; \mathbf q) \big(y - Q_{\mathrm{tot}}\big)^2,
\end{equation}
which corresponds precisely to the greedy policy loss. By the limit transfer theorem (Rockafellar and Wets, \emph{Variational Analysis}, Theorem~9.13), we obtain
\begin{equation}
\lim_{k \to \infty} \nabla f_k(\mathbf q)
= \lim_{k \to \infty} \nabla_{\mathbf q} \mathcal L_{\tau_k}(\mathbf q)
\;\in\; \partial^C f_0(\mathbf q).
\end{equation}

This completes the proof.

\end{proof}

\subsection{Proof of Theorem \ref{thm:2}}

\setcounter{theorem}{1}
\begin{theorem}[Stability of IGM-Consistent Fixed Points]\label{thm:2}
Suppose the following conditions hold:

(H1) For every $\mathbf{q} \in \mathcal{M}$, each agent's greedy action $\arg\max_{a_i} Q_i(a_i)$ is unique.

(H2) The Jacobian of the mixing function $f_{\mathrm{mix}}$ has full rank on the normal subspace of $\mathcal{M}$ (local non-degeneracy).

(H3) The softmax temperature $\tau > 0$ is sufficiently small such that the induced policy is dominated by greedy actions.

Then for any zero-loss fixed point $\mathbf{q}^* \in \mathcal{M}$ that satisfies IGM consistency (i.e., $\mathbf{g}(\mathbf{q}^*) = \mathbf{u}^*$), the Hessian $H_\tau(\mathbf{q}^*)$ is positive definite on the normal subspace $(T_{\mathbf{q}^*} \mathcal{M})^\perp$.
\end{theorem}

\begin{proof}
To establish the local stability of IGM-consistent zero-loss points, we demonstrate that for any perturbation direction $\mathbf{v}$ attempting to change the greedy action from $\mathbf{u}^*$ to some suboptimal action $\mathbf{g}'$, the quadratic form
\begin{equation}
    \mathbf{v}^\top H_\tau(\mathbf{q}^*) \mathbf{v}
\end{equation}
is strictly positive.

Without loss of generality, consider agent~1 and a suboptimal action $u'_1 \neq u_1^*$. Define the perturbation direction as
\begin{equation}
    \mathbf{v} = \mathbf{e}_{1u'_1} - \mathbf{e}_{1u_1^*},
\end{equation}
and examine the perturbed parameter vector
\begin{equation}
    \mathbf{q}(\epsilon) = \mathbf{q}^* + \epsilon \mathbf{v},
\end{equation}
for sufficiently small $\epsilon > 0$.

Since $\mathbf{q}^*$ is a stationary point with zero gradient, a second-order Taylor expansion of the loss function $\mathcal{L}_\tau$ around $\mathbf{q}^*$ yields
\begin{equation}
\mathcal{L}_\tau(\mathbf{q}(\epsilon)) 
= \frac{\epsilon^2}{2}\mathbf{v}^\top H_\tau(\mathbf{q}^*)\mathbf{v} + \mathcal{O}(\epsilon^3).
\end{equation}

Thus, proving local stability reduces to showing that 
$\mathbf{v}^\top H_\tau(\mathbf{q}^*)\mathbf{v} > 0$, 
which would imply the perturbed loss increases quadratically with respect to $\epsilon$ along any direction that attempts to replace the greedy action with a suboptimal one.

Under the low-temperature limit ($\tau \to 0$), the perturbed greedy action of agent~1 switches from $u^*_1$ to $u'_1$, leading to a new joint greedy action $\mathbf{g}' = (u'_1, u^*_2, \dots, u^*_N)$. The loss is then dominated by the error associated with $\mathbf{g}'$:
\begin{equation}
\mathcal{L}_\tau(\mathbf{q}(\epsilon)) \approx 
\left[y(\mathbf{g}') - Q_{\mathrm{tot}}(\mathbf{g}'; \mathbf{q}(\epsilon))\right]^2.
\end{equation}

Since $\mathbf{q}^* \in \mathcal{M}$, we have $Q_{\mathrm{tot}}(\mathbf{g}';\mathbf{q}^*) = y(\mathbf{g}')$. Using a first-order Taylor expansion around $\mathbf{q}^*$, we obtain
\begin{equation}
Q_{\mathrm{tot}}(\mathbf{g}'; \mathbf{q}(\epsilon)) \approx 
y(\mathbf{g}') + \epsilon \cdot \nabla_{\mathbf{q}} Q_{\mathrm{tot}}(\mathbf{g}';\mathbf{q}^*)^\top \mathbf{v}.
\end{equation}

Substituting this approximation into the loss expression, we get
\begin{equation}
\mathcal{L}_\tau(\mathbf{q}(\epsilon)) 
\approx \epsilon^2 \left(\nabla_{\mathbf{q}} Q_{\mathrm{tot}}(\mathbf{g}';\mathbf{q}^*)^\top \mathbf{v}\right)^2.
\end{equation}

Thus, by definition of the Hessian, we have
\begin{equation}
\mathbf{v}^\top H_\tau(\mathbf{q}^*)\mathbf{v} 
= 2\left(\nabla_{\mathbf{q}} Q_{\mathrm{tot}}(\mathbf{g}';\mathbf{q}^*)^\top \mathbf{v}\right)^2.
\end{equation}

Finally, by assumption (H2), the local non-degeneracy condition implies
\begin{equation}
\nabla_{\mathbf{q}} Q_{\mathrm{tot}}(\mathbf{g}';\mathbf{q}^*)^\top \mathbf{v} \neq 0.
\end{equation}
Hence, we conclude
\begin{equation}
\mathbf{v}^\top H_\tau(\mathbf{q}^*)\mathbf{v} > 0,
\end{equation}
proving the strict positive definiteness of the Hessian along any action-switching direction, and thereby establishing the local stability of the IGM-consistent fixed point $\mathbf{q}^*$.
\end{proof}

\subsection{Proof of Theorem \ref{thm:3}}

\begin{theorem}[Instability of IGM-Inconsistent Fixed Points]\label{thm:3}
Under the same assumptions (H1) -- (H3) as in Theorem~\ref{thm:3}, consider a zero-loss fixed point $\mathbf{q}^* \in \mathcal{M}$ that violates IGM consistency (i.e., $\mathbf{g}(\mathbf{q}^*) \neq \mathbf{u}^*$).
Then there exists an agent $i$ and a perturbation direction
\begin{equation}
\mathbf{v} = \mathbf{e}_{i, u_i^*} - \mathbf{e}_{i, g_i(\mathbf{q}^*)}
\end{equation}
such that
\begin{equation}
\mathbf{v}^\top H_\tau(\mathbf{q}^*) \mathbf{v}
\approx -\frac{2}{\tau} \left[ y(\mathbf{u}^*) - y(\mathbf{g}(\mathbf{q}^*)) \right] < 0,
\end{equation}
indicating that $H_\tau(\mathbf{q}^*)$ has a negative eigenvalue, and $\mathbf{q}^*$ is a structurally unstable saddle point.
\end{theorem}

\begin{proof}
To prove instability, it suffices to identify a perturbation direction $\mathbf{v}$ along which the Hessian exhibits negative curvature. Specifically, we aim to show that there exists at least one agent $i$ for which perturbing the greedy action from the current suboptimal choice $g_i(\mathbf{q}^*)$ to the globally optimal choice $u_i^*$ results in a strictly negative quadratic form:
\begin{equation}
    \mathbf{v}^\top H_\tau(\mathbf{q}^*) \mathbf{v} < 0.
\end{equation}

Since $\mathbf{g}(\mathbf{q}^*) \neq \mathbf{u}^*$, at least one agent's greedy action must differ from the globally optimal action. Without loss of generality, consider agent~$i$ with $g_i(\mathbf{q}^*) \neq u_i^*$, and define the perturbation direction as
\begin{equation}
    \mathbf{v} = \mathbf{e}_{i, u_i^*} - \mathbf{e}_{i, g_i(\mathbf{q}^*)}.
\end{equation}
This perturbation attempts to adjust the greedy action of agent~$i$ toward its globally optimal action $u_i^*$, yielding a corrected joint action 
\begin{equation}
\mathbf{g}' = (u_i^*, g_{-i}(\mathbf{q}^*)).
\end{equation}

In the low-temperature limit ($\tau \to 0$), the loss curvature along this direction is dominated by the payoff difference induced by switching actions. Specifically, a second-order Taylor expansion around $\mathbf{q}^*$ gives
\begin{equation}
    \mathbf{v}^\top H_\tau(\mathbf{q}^*) \mathbf{v} \approx -\frac{C}{\tau}\left[y(\mathbf{g}') - y(\mathbf{g}(\mathbf{q}^*))\right],
\end{equation}
where $C > 0$ is a positive constant depending on the policy gradient terms of the softmax function.

Since $\mathbf{g}(\mathbf{q}^*)$ is not globally optimal, we have
\begin{equation}
y(\mathbf{g}') > y(\mathbf{g}(\mathbf{q}^*)),
\end{equation}
implying
\begin{equation}
\mathbf{v}^\top H_\tau(\mathbf{q}^*) \mathbf{v} < 0.
\end{equation}
Thus, there exists a perturbation direction along which the Hessian has negative curvature, demonstrating the existence of a negative eigenvalue. Equivalently, the corresponding Jacobian of the learning dynamics possesses a positive eigenvalue, establishing that the IGM-inconsistent fixed point $\mathbf{q}^*$ is a structurally unstable saddle point.

\end{proof}

% Appendix B
\section{Implementation Details}
\label{app:impl}
\subsection{Model Architecture}
Our model's architecture, as illustrated in Figure \ref{fig:qmix}, is based on the QMIX \citep{rashid2018qmix} framework and is composed of a set of decentralized agent networks and a centralized mixing network.

\paragraph{Agent Networks}
Each agent possesses an individual recurrent network responsible for approximating its local action-value function, $Q_i(\tau^i, a_t^i)$. This network processes the agent's current local observation ($o_t^i$) and its previous action ($a_{t-1}^i$). The input is first passed through a Multi-Layer Perceptron (MLP), and the output is then fed into a Gated Recurrent Unit (GRU) which also receives the previous hidden state $h_{t-1}^i$. The GRU's new hidden state, $h_t^i$, which encodes the agent's action-observation history $\tau^i$, is then passed to a final MLP to produce the Q-values for all of the agent's discrete actions. During execution, actions are chosen from these local Q-values using an $\epsilon$-greedy policy.

\paragraph{Mixing Network}
The mixing network is a centralized module that combines the individual Q-values from all agent networks, $\{Q_1(\tau^1, a_t^1), \dots, Q_n(\tau^n, a_t^n)\}$, into a single global joint action-value, $Q_{tot}(\tau, a)$. This network also takes the global state $s_t$ as an input to enable state-dependent mixing.

\paragraph{Hypernetwork Structure}
The mixing network utilizes a hypernetwork architecture. It consists of several linear layers that take the global state $s_t$ as input and generate the weights (e.g., $w_1, w_2$) and biases (e.g., $b_1, b_2$) for the main mixing layers. The individual agent Q-values are then fed through these state-generated layers, with an ELU activation function applied between them, to compute the final $Q_{tot}$. For standard QMIX, the weights of the hypernetwork are constrained to be non-negative to enforce the monotonicity constraint $\frac{\partial Q_{tot}}{\partial Q_i} \ge 0$. Our non-monotonic approach, central to this paper, removes this constraint.

\begin{figure}
    \centering
    \includegraphics[width=1.0\linewidth]{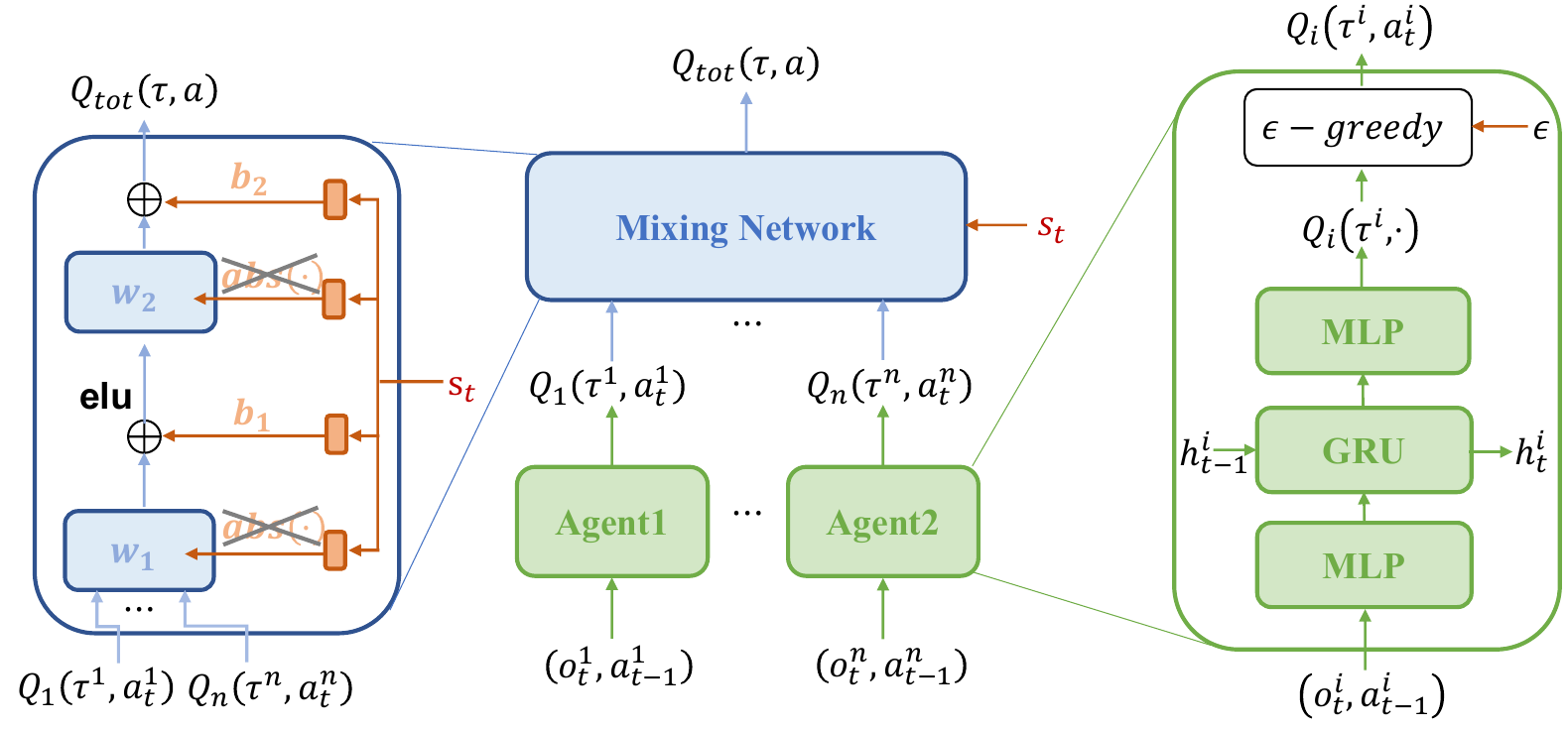}
    \caption{Architecture of our method.}
    \label{fig:qmix}
\end{figure}

\begin{figure}
    \centering
    \includegraphics[width=1.0\linewidth]{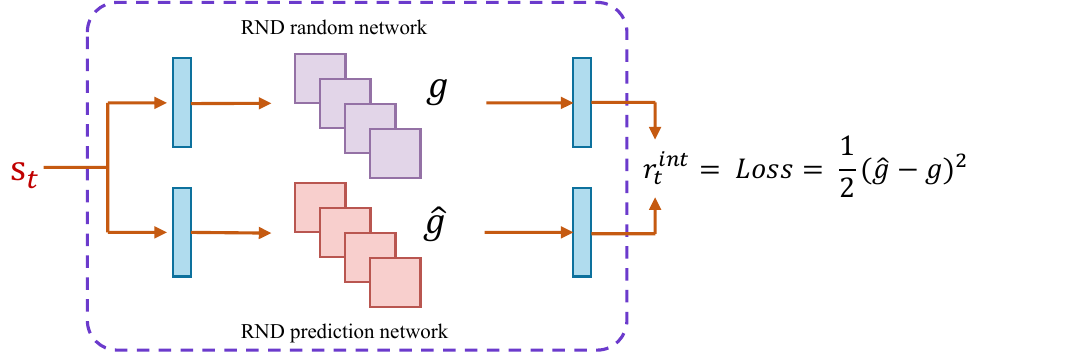}
    \caption{Architecture of RND.}
    \label{fig:rnd}
\end{figure}

\subsection{Intrinsic Reward Generation for Exploration}
To facilitate robust exploration, especially in complex environments with sparse extrinsic rewards, we integrate an intrinsic curiosity mechanism based on Random Network Distillation (RND) \citep{rnd}. As illustrated in Figure \ref{fig:rnd}, the RND module consists of two distinct neural networks:
\begin{itemize}
    \item A \textbf{target network} ($g$), which is initialized randomly and remains fixed throughout training. It takes the current state $s_t$ and maps it to a feature embedding.
    \item A \textbf{predictor network} ($\hat{g}$), which is trained to predict the output of the target network for the same state $s_t$.
\end{itemize}
The prediction error between the outputs of these two networks serves as the intrinsic reward signal. Specifically, the intrinsic reward $r_t^{int}$ is calculated as the mean squared error between the feature embeddings:
\begin{equation}
    r_t^{int} = \frac{1}{2} \| \hat{g}(s_t; \theta_{\text{predictor}}) - g(s_t) \|_2^2
\end{equation}
where $\theta_{\text{predictor}}$ denotes the parameters of the trainable predictor network. This reward is high for novel, unfamiliar states where the predictor network makes large errors, and low for familiar states that have been visited frequently. By adding this curiosity-driven reward to the extrinsic environment reward, we encourage the agents to systematically explore their state space and discover more promising regions of the environment.

\subsection{Pseudocode}

\begin{algorithm}[htbp]
\caption{Training Procedure}
\label{alg:our_method_td_lambda}
\renewcommand{\algorithmicrequire}{\textbf{Initialize:}}
\begin{algorithmic}[1]
    \REQUIRE Agent networks $Q_i(\cdot;\theta_i)$ for $i=1..N$, and mixing network $Q_{tot}(\cdot;\theta_{mix})$
    \REQUIRE Target networks $\theta_i^-, \theta_{mix}^-$ with parameters copied from online networks
    \REQUIRE RND target network $g$ (fixed) and predictor network $\hat{g}(\cdot;\theta_{\text{predictor}})$
    \REQUIRE Replay buffer $D \leftarrow \emptyset$
    
    \FOR{episode = $1, \dots, M$}
        \STATE Initialize environment and get initial state $s_0$
        \FOR{$t=0, \dots, T-1$}
            \FOR{agent $i=1, \dots, N$}
                \STATE Based on observation history $\tau_t^i$, select action $a_t^i$ using an $\epsilon$-greedy policy on $Q_i(\tau_t^i, \cdot; \theta_i)$
            \ENDFOR
            \STATE Execute joint action $\mathbf{a}_t$, observe extrinsic reward $r_t$, next state $s_{t+1}$, and terminated flag $d_{t+1}$
            \STATE Store transition sequence in buffer $D$
            
            \IF{time to update}
                \STATE Sample a random mini-batch of trajectories $B$ from $D$
                \STATE Initialize loss $\mathcal{L} \leftarrow 0$
                \STATE For each trajectory $\tau_j = \{(s_t, \mathbf{a}_t, r_t, d_t)\}_{t=0...T}$ in $B$:
                \STATE \quad // First, calculate total rewards for the trajectory
                \STATE \quad For $t = 0, \dots, T$:
                \STATE \quad \quad $r_{int, t} \leftarrow \| \hat{g}(s_{t+1}; \theta_{\text{predictor}}) - g(s_{t+1}) \|_2^2$
                \STATE \quad \quad $r_{total, t} \leftarrow r_t + \beta \cdot r_{int, t}$
                \STATE \quad For $t = 0, \dots, T$:
                \STATE \quad \quad // Calculate n-step SARSA returns starting from time t
                \STATE \quad \quad For $n=1, \dots, T-t$:
                \STATE \quad \quad \quad $y_t^{(n)} \leftarrow \left( \sum_{k=0}^{n-1} \gamma^k r_{total, t+k} \right) + (1-d_{t+n})\gamma^n Q_{tot}(s_{t+n}, \mathbf{a}_{t+n}; \theta_{mix}^-, \theta_i^-)$
                \STATE \quad \quad // Combine n-step returns to form the TD($\lambda$) target
                \STATE \quad \quad $y_t^{\lambda} \leftarrow (1-\lambda)\sum_{n=1}^{T-t-1} \lambda^{n-1} y_t^{(n)} + \lambda^{T-t-1}y_t^{(T-t)}$
                \STATE \quad \quad $Q_{tot, t} \leftarrow Q_{tot}(s_t, \mathbf{a}_t; \theta_{mix})$
                \STATE \quad \quad $\mathcal{L} \leftarrow \mathcal{L} + (y_t^{\lambda} - Q_{tot, t})^2$
                
                \STATE Update RND predictor network $\theta_{\text{predictor}}$ to minimize $\sum_{j,t} r_{int, t}$
                \STATE Update agent networks $\theta_i$ and mixing network $\theta_{mix}$ by minimizing $\mathcal{L}$
                \STATE Periodically update target networks: $\theta_i^- \leftarrow \theta_i$, $\theta_{mix}^- \leftarrow \theta_{mix}$
            \ENDIF
        \ENDFOR
    \ENDFOR
\end{algorithmic}
\end{algorithm}

The complete training procedure of our proposed method is summarized in Algorithm~\ref{alg:our_method_td_lambda}. 
We first initialize all agent networks, the non-monotonic mixing network, and their corresponding target networks, which are cloned from the initial parameters. 
In parallel, we initialize the RND module, consisting of a fixed random target network and a trainable predictor network, as well as an empty replay buffer for experience storage.

During data collection, agents interact with the environment over full episodes. 
At each timestep, every agent selects an action according to its local action-value estimates using an $\epsilon$-greedy exploration policy. 
The resulting joint action, extrinsic reward, and next state are then stored in the replay buffer as part of the trajectory.

The training phase is executed periodically. We sample a mini-batch of complete trajectories from the replay buffer. For each step within these trajectories, we first compute an intrinsic reward using the RND module's prediction error. This is added to the extrinsic reward to form a total reward. Next, we calculate the target value for the Bellman update. Following our methodology, we use the TD($\lambda$) return, which forms a robust target by calculating a weighted average of all n-step SARSA returns along the trajectory. The final loss is the mean squared error between these TD($\lambda$) targets and the Q-values produced by the online mixing network. This loss is used to update the parameters of both the agent networks and the mixing network via gradient descent. The RND predictor network is separately updated to minimize its prediction error. Finally, the target networks $\theta^-$ are periodically updated by copying the parameters from the online networks $\theta$.

\subsection{Hyperparameters}
The important parameters used in the experiment are presented in Table 1\ref{tab:hyper}. Mini-batch size refers to the number of episodes sampled from the buffer for training. Annual start weight means that the initial value of $\beta$ is 0.5 and reduces to 0.05 over 100k steps.

\begin{table}[htbp]
    \centering
    \caption{Hyperparameters}
    \resizebox{0.8\linewidth}{!}{
    \begin{tabular}{@{}llcll@{}}
    \toprule
    Hyperparameter & Value &  \phantom{} & Hyperparameter & Value \\ 
    \midrule
      Optimizer & Adam &&   Buffer size  &  5000  \\
      Agent learning rate &  1e-3 && RND learning rate & 5e-4 \\
      Mini-batch size  & 128 && Anneal start weight & 0.5 \\
      Anneal step for $\beta$ & 100K && Anneal finish weight & 0.05 \\
      Use learning rate decay & False   \\
    \bottomrule
    \end{tabular}}
    \label{tab:hyper}
\end{table}

\section{Experimental Validation and Ablation Studies}

In this section, we present additional experiments to further substantiate the theoretical results established in the main paper.

\subsection{Matrix Game Results}
\label{sec:appendix_matrix_game}

We provide the full set of results from the matrix-game experiments summarized in the main text. 
Table~\ref{tab:matrix_game} compares the learned value functions produced by our method, QTRAN \citep{son2019qtran}, QMIX\citep{rashid2018qmix}, and QPLEX\citep{wang2020qplex} against the true payoff matrices across two distinct games.

As the results indicate, both our unconstrained method and QTRAN accurately recovered the global payoff matrices. 
In contrast, QMIX and QPLEX, which enforce monotonic constraints, failed to represent the true payoffs, especially in the non-monotonic Game A scenario. 
These findings strongly support our theoretical claim that removing monotonicity constraints enables reliable recovery of IGM-optimal solutions in cases where monotonic approaches inherently fall short.

\begin{table}[htbp]
\centering
\begin{subtable}[t]{0.45\linewidth}
\centering
\begin{tabular}{|c|c|c|c|}
\hline
\diagbox{$u_1$}{$u_2$} & A & B & C \\
\hline
A & 12 & -12 & -12 \\
\hline
B & -12 & 0 & 0 \\
\hline
C & -12 & 0 & 0 \\
\hline
\end{tabular}
\caption{True payoff of matrix game A}
\end{subtable}
\hfill
\begin{subtable}[t]{0.45\linewidth}
\centering
\begin{tabular}{|c|c|c|c|}
\hline
\diagbox{$u_1$}{$u_2$} & A & B & C \\
\hline
A & 12 & 0 & 10 \\
\hline
B & 0 & 0 & 10 \\
\hline
C & 10 & 10 & 10 \\
\hline
\end{tabular}
\caption{True payoff of matrix game B}
\end{subtable}

\begin{subtable}[t]{0.45\linewidth}
\centering
\begin{tabular}{|c|c|c|c|}
\hline
\diagbox{$Q_1$}{$Q_2$} & A(7.4) & B(-4.0) & C(-4.1) \\
\hline
A(8.1) & 12.0 & -12.0 & -12.0 \\
\hline
B(-4.3) & -12.0 & 0.0 & 0.0 \\
\hline
C(-4.3) & -12.0 & 0.0 & 0.0 \\
\hline
\end{tabular}
\caption{Our method on game A}
\end{subtable}
\hfill
\begin{subtable}[t]{0.45\linewidth}
\centering
\begin{tabular}{|c|c|c|c|}
\hline
\diagbox{$Q_1$}{$Q_2$} & A(6.0) & B(-3.7) & C(2.5) \\
\hline
A(6.2) & 12.0 & 0.0 & 10.0 \\
\hline
B(-4.4) & 0.0 & 0.0 & 10.0 \\
\hline
C(1.7) & 10.0 & 10.0 &10.0 \\
\hline
\end{tabular}
\caption{Our method on game B}
\end{subtable}

\begin{subtable}[t]{0.45\linewidth}
\centering
\begin{tabular}{|c|c|c|c|}
\hline
\diagbox{$Q_1$}{$Q_2$} & A(-5.2) & B(0.1) & C(0.1) \\
\hline
A(-5.3) & -12.0 & -12.0 & -12.0 \\
\hline
B(0.1) & -12.0 & 0.0 & -0.1 \\
\hline
C(0.1) & -12.0 & -0.0 & -0.0 \\
\hline
\end{tabular}
\caption{QMIX on game A}
\end{subtable}
\hfill
\begin{subtable}[t]{0.45\linewidth}
\centering
\begin{tabular}{|c|c|c|c|}
\hline
\diagbox{$Q_1$}{$Q_2$} & A(-0.6) & B(-0.6) & C(0.5) \\
\hline
A(-0.6) & 0.1 & 0.1 & 4.1 \\
\hline
B(-0.6) & 0.1 & 0.1 & 4.1 \\
\hline
C(0.5) & 5.6 & 5.6 & 10.0 \\
\hline
\end{tabular}
\caption{QMIX on game B}
\end{subtable}

\begin{subtable}[t]{0.45\linewidth}
\centering
\begin{tabular}{|c|c|c|c|}
\hline
\diagbox{$Q_1$}{$Q_2$} & A(2.0) & B(-8.5) & C(-8.9) \\
\hline
A(1.9) & 12.0 & -17.8 & -17.4 \\
\hline
B(-8.7) & -14.6 & -44.4 & -44.0 \\
\hline
C(-8.8) & -15.1 & -44.6 & -44.9 \\
\hline
\end{tabular}
\caption{QPLEX on game A}
\end{subtable}
\hfill
\begin{subtable}[t]{0.45\linewidth}
\centering
\begin{tabular}{|c|c|c|c|}
\hline
\diagbox{$Q_1$}{$Q_2$} & A(0.4) & B(0.0) & C(-27.1) \\
\hline
A(0.4) & 12.1 & 6.9 & -347.0 \\
\hline
B(0.0) & 5.3 & 0.2 & -353.7 \\
\hline
C(-26.9) & -430.9 & -436.0 & -789.9 \\
\hline
\end{tabular}
\caption{QPLEX on game B}
\end{subtable}

\begin{subtable}[t]{0.45\linewidth}
\centering
\begin{tabular}{|c|c|c|c|}
\hline
\diagbox{$Q_1$}{$Q_2$} & A(6.3) & B(0.2) & C(0.7) \\
\hline
A(6.0) & 12.0 & -12.0 & -12.0 \\
\hline
B(0.1) & -12.0 & 0.0 & 0.0 \\
\hline
C(0.6) & -12.0 & 0.0 & 0.0 \\
\hline
\end{tabular}
\caption{QTRAN on game A}
\end{subtable}
\hfill
\begin{subtable}[t]{0.45\linewidth}
\centering
\begin{tabular}{|c|c|c|c|}
\hline
\diagbox{$Q_1$}{$Q_2$} & A(4.5) & B(3.3) & C(3.6) \\
\hline
A(4.0) & 12.0 & 0.0 & 10.0 \\
\hline
B(2.7) & -0.0 & 0.0 & 10.0 \\
\hline
C(3.4) & 10.0 & 10.0 & 10.0 \\
\hline
\end{tabular}
\caption{QTRAN on game B}
\end{subtable}
\caption{True payoff matrices and estimated value functions for two matrix games. Each column corresponds to one game: the first row shows the ground-truth payoff matrix of Game A and Game B, and each of the remaining rows represents a method's value estimate.}
\label{tab:matrix_game}
\end{table}

\begin{figure*}[!h]
\centering
\includegraphics[width=1.0\textwidth]{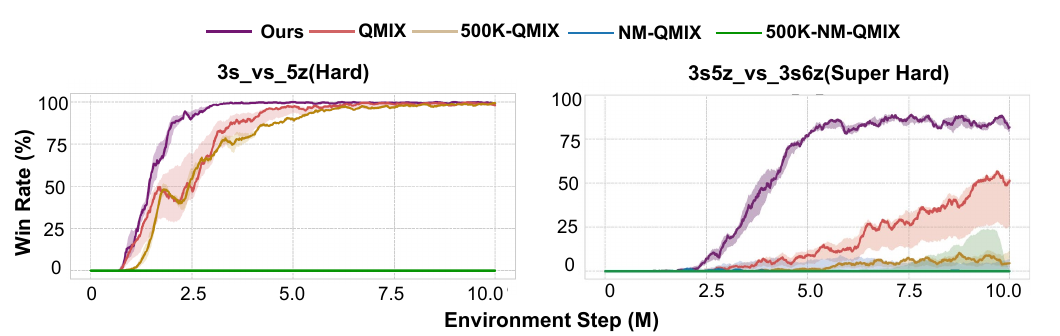}
\caption{Comparisons of test win rate on SMAC maps: \texttt{3s\_vs\_5z}, and \texttt{3s5z\_vs\_3s6z}. The results are averaged over five independent runs, with the 25\%–75\% interquartile range shown as a shaded region.}
\label{fig:ablation_smac_1}
\end{figure*}

\begin{figure*}[!h]
\centering
\includegraphics[width=1.0\textwidth]{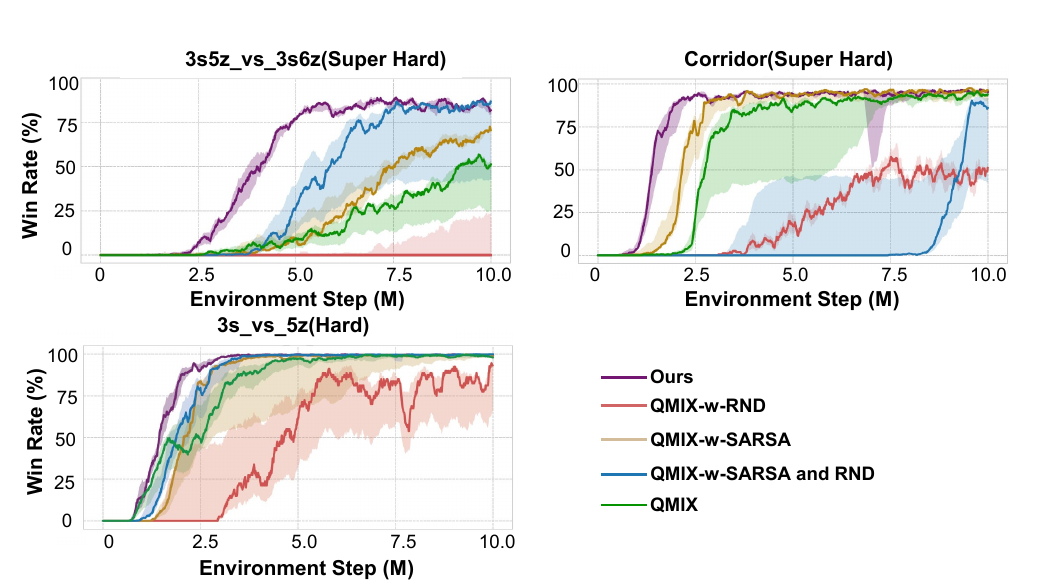}
\caption{Comparisons of test win rate on SMAC maps: \texttt{3s\_vs\_5z}, \texttt{corridor}, and \texttt{3s5z\_vs\_3s6z}. The results are averaged over five independent runs, with the 25\%–75\% interquartile range shown as a shaded region.}
\label{fig:ablation_smac_2}
\end{figure*}

\begin{figure*}[t]
\centering
\includegraphics[width=1.0\textwidth]{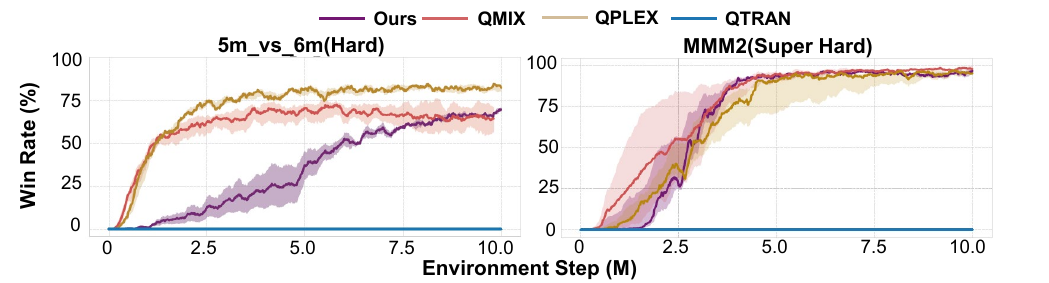}
\caption{Comparisons of test win rate on SMAC maps: \texttt{5m\_vs\_6m} and \texttt{MMM2}. The results are averaged over five independent runs, with the 25\%–75\% interquartile range shown as a shaded region.}
\label{fig:additional}
\end{figure*}

\subsection{Ablation Studies}

We conduct ablation studies to validate the key design choices of our proposed method. 

First, we investigate the effect of directly removing the monotonicity constraint from QMIX, denoted as NM-QMIX. As shown in Figure~\ref{fig:ablation_smac_1}, NM-QMIX performs very poorly on tasks such as \texttt{3s5z\_vs\_3s6z}. This indicates that the monotonicity constraint is indeed beneficial in most cases, despite its limited representational capacity. By default, QMIX uses an exploration annealing schedule of 50k steps; we extend this to 500k but observe no improvement. This suggests that simply increasing $\epsilon$-greedy exploration provides limited benefit. In contrast, our method achieves substantially better performance, which we attribute to the use of a more effective exploration strategy and the SARSA-style update. These findings align with our theoretical analysis: (1) sufficient exploration is crucial for escaping unstable saddle points and ultimately converging to the stable submanifold, and (2) the SARSA-style update, which does not rely on the $\max_{a’}$ operator, provides a more reliable learning signal.

Second, we examine the effect of adding the RND component and the SARSA-style update to the original QMIX. As shown in the Figure~\ref{fig:ablation_smac_2}, neither component consistently improves performance for standard QMIX. The SARSA-style update is specifically designed for non-monotonic factorization; adding it to monotonic QMIX is unnecessary. As for RND, it enhances exploration, but stronger exploration is not always beneficial: it helps our non-monotonic method recover IGM-optimal solutions, yet it tends to slow convergence in QMIX. These results further indicate that with an effective exploration mechanism and an appropriate algorithmic design, removing the monotonicity constraint can indeed lead to improved performance.

\subsection{Failure Cases}

We also observe that our method does not consistently outperform the baselines on a small number of environments, as shown in Figure~\ref{fig:additional}. In the \texttt{5m\_vs\_6m} environment, our method converges more slowly but eventually reaches performance comparable to QMIX. In \texttt{MMM2}, all methods exhibit very similar learning curves, except for QTRAN, and our method converges slightly more slowly than QMIX. Although we do not achieve absolute superiority on every benchmark, the results still demonstrate that multi-agent Q-learning can perform strongly even without imposing any constraints. We believe that further exploration of this unconstrained setting is a promising future direction.

\section{Limitations and Future Work}
\label{app:lim}

We provide several discussions on the limitations of our work and potential directions for future research.

First, our theoretical analysis of learning dynamics is conducted in a simplified setting. Although the empirical results on challenging sequential benchmarks such as SMAC and GRF indicate that the core insights generalize well, a rigorous extension of the stability analysis to the full multi-state Dec-POMDP setting remains an important avenue for future investigation.

We also discuss a possible direction for extending the theoretical analysis to practical settings. For RL based on multi-state MDPs, the overall training objective can be formulated as a weighted expectation of per-state losses. Consequently, the global Hessian can be regarded as the expectation of per-state Hessians. Therefore, the conclusion of Theorem 2 naturally extends to standard MDPs: if the quadratic form on the normal subspace remains positive for IGM-consistent points at each visited state, the expected Hessian will also be positive definite in the normal directions, implying that IGM-consistent points remain asymptotically stable.

Lastly, our approach adopts a SARSA-style update rule, which is theoretically on-policy. For practical reasons, and following prior work~\citep{hernandez2019understanding}, we omit off-policy corrections such as importance sampling. The implications of this omission, particularly the potential distribution shift between the replay buffer and the current policy, are not explicitly captured in our theoretical analysis and warrant further study.

\end{document}